\documentclass[10pt,twocolumn,letterpaper]{article}

\usepackage{iccv}
\usepackage{times}
\usepackage{epsfig}
\usepackage{graphicx}
\usepackage{amsmath}
\usepackage{amssymb}

\usepackage{amsthm}

\def \R {\mathbb{R}}

\def \x {\mathbf{x}}

\def \OO {\mathcal{O}}

\def \z {\mathbf{z}}

\def \p {\mathbf{p}}
\def \q {\mathbf{q}}

\def \q {\mathbf{q}}

\def \w {\mathbf{w}}

\newtheorem{thm}{Theorem}
\newtheorem{prop}{Proposition}

\usepackage[pagebackref=true,breaklinks=true,letterpaper=true,colorlinks,bookmarks=false]{hyperref}

\iccvfinalcopy 

\def\iccvPaperID{5183} 

\ificcvfinal\fi
\begin{document}

\title{SoftTriple Loss: Deep Metric Learning Without Triplet Sampling}

\author{Qi Qian$^{1}$ \quad Lei Shang$^{2}$ \quad Baigui Sun$^{2}$\quad Juhua Hu$^{3}$\quad Hao Li$^{2}$\quad Rong Jin$^{1}$\\
$^{1}$ Alibaba Group, Bellevue, WA, 98004, USA\\
$^{2}$ Alibaba Group, Hangzhou, China\\
$^{3}$ School of Engineering and Technology\\
University of Washington, Tacoma, WA, 98402, USA\\
{\tt\small \{qi.qian, sl172005, baigui.sbg, lihao.lh, jinrong.jr\}@alibaba-inc.com, juhuah@uw.edu}
}

\maketitle

\begin{abstract}
Distance metric learning (DML) is to learn the embeddings where examples from the same class are closer than examples from different classes. It can be cast as an optimization problem with triplet constraints. Due to the vast number of triplet constraints, a sampling strategy is essential for DML. With the tremendous success of deep learning in classifications, it has been applied for DML. When learning embeddings with deep neural networks (DNNs), only a mini-batch of data is available at each iteration. The set of triplet constraints has to be sampled within the mini-batch. Since a mini-batch cannot capture the neighbors in the original set well, it makes the learned embeddings sub-optimal. On the contrary, optimizing SoftMax loss, which is a classification loss, with DNN shows a superior performance in certain DML tasks. It inspires us to investigate the formulation of SoftMax. Our analysis shows that SoftMax loss is equivalent to a smoothed triplet loss where each class has a single center. In real-world data, one class can contain several local clusters rather than a single one, e.g., birds of different poses. Therefore, we propose the SoftTriple loss to extend the SoftMax loss with multiple centers for each class. Compared with conventional deep metric learning algorithms, optimizing SoftTriple loss can learn the embeddings without the sampling phase by mildly increasing the size of the last fully connected layer. Experiments on the benchmark fine-grained data sets demonstrate the effectiveness of the proposed loss function. Code is available at \url{https://github.com/idstcv/SoftTriple}.
\end{abstract}

\section{Introduction}
Distance metric learning (DML) has been extensively studied in the past decades due to its broad range of applications, e.g., $k$-nearest neighbor classification~\cite{WeinbergerS09}, image retrieval~\cite{SongXJS16} and clustering~\cite{XingNJR02}. 
\begin{figure}[!ht]
\centering
\includegraphics[width=3.2in]{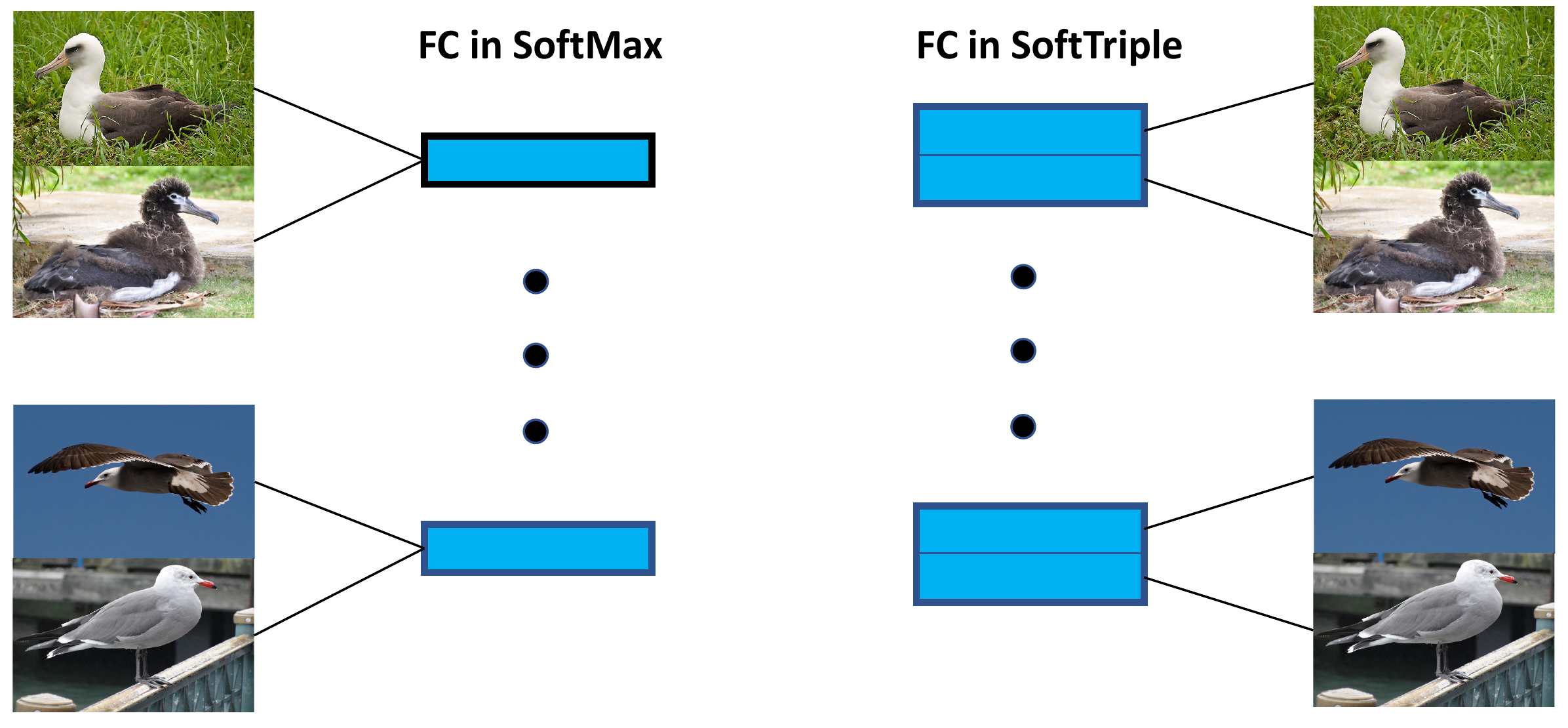}
\caption{Illustration of the proposed SoftTriple loss. In conventional SoftMax loss, each class has a representative center in the last fully connected layer. Examples in the same class will be collapsed to the same center. It may be inappropriate for the real-world data as illustrated. In contrast, SoftTriple loss keeps multiple centers (e.g., 2 centers per class in this example) in the fully connected layer and each image will be assigned to one of them. It is more flexible for modeling intra-class variance in real-world data sets.\label{fig:illu}}
\end{figure}
With an appropriate distance metric, examples from the same class should be closer than examples from different classes. Many algorithms have been proposed to learn a good distance metric~\cite{ParkhiVZ15,QianJY0Z15,SchroffKP15,WeinbergerS09}.

In most of conventional DML methods, examples are represented by hand-crafted features, and DML is to learn a feature mapping to project examples from the original feature space to a new space. The distance can be computed as the Mahalanobis distance~\cite{mahalanobis1936generalized}
\[\mathrm{dist}_M(\x_i,\x_j) = (\x_i-\x_j)^\top M(\x_i-\x_j)\]
where $M$ is the learned distance metric. With this formulation, the main challenge of DML is from the dimensionality of input space. As a metric, the learned matrix $M$ has to be positive semi-definite (PSD) while the cost of keeping the matrix PSD can be up to $\OO(d^3)$, where $d$ is the dimensionality of original features. The early work directly applies PCA to shrink the original space~\cite{WeinbergerS09}. Later, various strategies are developed to reduce the computational cost~\cite{QianJY0Z15,QianJZL15}.

Those approaches can obtain the good metric from the input features, but the hand-crafted features are task independent and may cause the loss of information, which limits the performance of DML. With the success of deep neural networks in classification~\cite{KrizhevskySH12}, researchers consider to learn the embeddings directly from deep neural networks~\cite{ParkhiVZ15,SchroffKP15}. Without the explicit feature extraction, deep metric learning boosts the performance by a large margin~\cite{SchroffKP15}. In deep metric learning, the dimensionality of input features is no longer a challenge since neural networks can learn low-dimensional features directly from raw materials, e.g., images, documents, etc. In contrast, generating appropriate constraints for optimization becomes challenging for deep metric learning.

It is because most of deep neural networks are trained with the stochastic gradient descent (SGD) algorithm and only a mini-batch of examples are available at each iteration. Since embeddings are optimized with the loss defined on an anchor example and its neighbors (e.g., the active set of pairwise~\cite{XingNJR02} or triplet~\cite{WeinbergerS09} constraints), the examples in a mini-batch may not be able to capture the overall neighborhood well, especially for relatively large data sets. Moreover, a mini-batch contains $\OO(m^2)$ pairs and $\OO(m^3)$ triplets, where $m$ is the size of the mini-batch. An effective sampling strategy over the mini-batch is essential even for a small batch (e.g., $32$) to learn the embeddings efficiently. Many efforts have been devoted to studying sampling an informative mini-batch~\cite{RippelPDB15,SchroffKP15} and sampling triplets within a mini-batch~\cite{ManmathaWSK17,SongXJS16}. Some work also tried to reduce the total number of triplets with proxies~\cite{Attias17,QianTLZJ18}. The sampling phase for the mini-batch and constraints not only loses the information but also makes the optimization complicated. In this work, we consider to learn embeddings without constraints sampling.

Recently, researches have shown that embeddings obtained directly from optimizing SoftMax loss, which is proposed for classification, perform well on the simple distance based tasks~\cite{Sohn16,WenZL016} and face recognition~\cite{07698, LiuWYLRS17,LiuWYY16,WangCLL18,WangWZJGZL018}. It inspires us to investigate the formulation of SoftMax loss. Our Analysis demonstrates that SoftMax loss is equivalent to a smoothed triplet loss. By providing a single center for each class in the last fully connected layer, the triplet constraint derived by SoftMax loss can be defined on an original example, its corresponding center and a center from a different class. Therefore, embeddings obtained by optimizing SoftMax loss can work well as a distance metric. However, a class in real-world data can consist of multiple local clusters as illustrated in Fig.~\ref{fig:illu} and a single center is insufficient to capture the inherent structure of the data. Consequently, embeddings learned from SoftMax loss can fail in the complex scenario~\cite{Sohn16}.

In this work, we propose to improve SoftMax loss by introducing multiple centers for each class and the novel loss is denoted as SoftTriple loss. Compared with a single center, multiple ones can capture the hidden distribution of the data better due to the fact that they help to reduce the intra-class variance. This property is also crucial to reserve the triplet constraints over original examples while training with multiple centers. Compared with existing deep DML methods, the number of triplets in SoftTriple is linear in the number of original examples. Since the centers are encoded in the last fully connected layer, SoftTriple loss can be optimized without sampling triplets. Fig.~\ref{fig:illu} illustrates the proposed SoftTriple loss. Apparently, SoftTriple loss has to determine the number of centers for each class. To alleviate this issue, we develop a strategy that sets a sufficiently large number of centers for each class at the beginning and then applies $L_{2,1}$ norm to obtain a compact set of centers. We demonstrate the proposed loss on the fine-grained visual categorization tasks, where capturing local clusters is essential for good performance~\cite{QianJZL15}.

The rest of this paper is organized as follows. Section \ref{sec:related} reviews the related work of conventional distance metric learning and deep metric learning. Section \ref{sec:method} analyzes the SoftMax loss and proposes the SoftTriple loss accordingly. Section \ref{sec:exp} conducts comparisons on benchmark data sets. Finally, Section \ref{sec:conclud} concludes this work and discusses future directions.

\section{Related Work}
\label{sec:related}
\paragraph{Distance metric learning}
Many DML methods have been developed when input features are provided~\cite{WeinbergerS09,XingNJR02}. The dimensionality of input features is a critical challenge for those methods due to the PSD projection, and many strategies have been proposed to alleviate it. The most straightforward way is to reduce the dimension of input space by PCA~\cite{WeinbergerS09}. However, PCA is task independent and may hurt the performance of learned embeddings. Some works try to reduce the number of valid parameters with the low-rank assumption~\cite{LimLM13}. \cite{QianJY0Z15} decreases the computational cost by reducing the number of PSD projections. \cite{QianJZL15} proposes to learn the dual variables in the low-dimensional space introduced by random projections and then recover the metric in the original space. After addressing the challenge from the dimensionality, the hand-crafted features become the bottleneck of performance improvement. 

The forms of constraints for metric learning are also developed in these methods. Early work focuses on optimizing pairwise constraints, which require the distances between examples from the same class small while those from different classes large~\cite{XingNJR02}. Later, \cite{WeinbergerS09} develops the triplet constraints, where given an anchor example, the distance between the anchor point and a similar example should be smaller than that between the anchor point and a dissimilar example by a large margin. It is obvious that the number of pairwise constraints is $\OO(n^2)$ while that of triplet constraints can be up to $\OO(n^3)$, where $n$ is the number of original examples. Compared with the pairwise constraints, triplet constraints optimize the geometry of local cluster and are more applicable for modeling intra-class variance. In this work, we will focus on the triplet constraints.

\paragraph{Deep metric learning}
Deep metric learning aims to learn the embeddings directly from the raw materials (e.g., images) by deep neural networks~\cite{ParkhiVZ15,SchroffKP15}. With the task dependent embeddings, the performance of metric learning has a dramatical improvement. However, most of deep models are trained with SGD that allows only a mini-batch of data at each iteration. Since the size of mini-batch is small, the information in it is limited compared to the original data. To alleviate this problem, algorithms have to develop an effective sampling strategy to generate the mini-batch and then sample triplet constraints from it. A straightforward way is increasing the size of mini-batch~\cite{SchroffKP15}. However, the large mini-batch will suffer from the GPU memory limitation and can also increase the challenge of sampling triplets. Later, \cite{RippelPDB15} proposes to generate the mini-batch from neighbor classes. Besides, there are various sampling strategies for obtaining constraints~\cite{GeHDS18,ManmathaWSK17,SchroffKP15,SongXJS16}. \cite{SchroffKP15} proposes to sample the semi-hard negative examples. \cite{SongXJS16} adopts all negative examples within the margin for each positive pair. \cite{ManmathaWSK17} develops distance weighted sampling that samples examples according to the distance from the anchor example. \cite{GeHDS18} selects hard triplets with a dynamic violate margin from a hierarchical class-level tree. However, all of these strategies may fail to capture the distribution of the whole data set. Moreover, they make the optimization in deep DML complicated. 

\paragraph{Learning with proxies}
Recently, some researchers consider to reduce the total number of triplets to alleviate the challenge from the large number of triplets. \cite{Attias17} constructs the triplet loss with one original example and two proxies. Since the number of proxies is significantly less than the number of original examples, proxies can be kept in the memory that help to avoid the sampling over different batches. However, it only provides a single proxy for each class when label information is available, which is similar to SoftMax. \cite{QianTLZJ18} proposes a conventional DML algorithm to construct the triplet loss only with latent examples, which assigns multiple centers for each class and further reduces the number of triplets. In this work, we propose to learn the embeddings by optimizing the proposed SoftTriple loss to eliminate the sampling phase and capture the local geometry of each class simultaneously.

\section{SoftTriple Loss}
\label{sec:method}
In this section, we first introduce the SoftMax loss and the triplet loss and then study the relationship between them to derive the SoftTriple loss. 

Denote the embedding of the $i$-th example as $\x_i$ and the corresponding label as $y_i$, then the conditional probability output by a deep neural network can be estimated via the SoftMax operator
\[\Pr(Y=y_i|\x_i) = \frac{\exp(\w_{y_i}^\top \x_i)}{\sum_j^C \exp(\w_j^\top \x_i)}\]
where $[\w_1,\cdots,\w_C]\in\R^{d\times C}$ is the last fully connected layer. $C$ denotes the number of classes and $d$ is the dimension of embeddings.
The corresponding SoftMax loss is
\[\ell_{\mathrm{SoftMax}}(\x_i) = -\log\frac{\exp(\w_{y_i}^\top \x_i)}{\sum_j \exp(\w_j^\top \x_i)}\]
A deep model can be learned by minimizing losses over examples. This loss has been prevalently applied for classification task~\cite{KrizhevskySH12}.

Given a triplet $(\x_i, \x_j, \x_k)$, DML aims to learn good embeddings such that examples from the same class are closer than examples from different classes, i.e.,
\[\forall i,j,k,\quad \|\x_i- \x_k\|_2^2 - \|\x_i- \x_j\|_2^2 \geq \delta\]
where $\x_i$ and $\x_j$ are from the same class and $\x_k$ is from a different class. $\delta$ is a predefined margin. When each example has the unit length (i.e., $\|\x\|_2=1$), the triplet constraint can be simplified as
\begin{eqnarray}\label{eq:const}
\forall i,j,k, \quad \x_i^\top \x_j - \x_i^\top \x_k \geq \delta
\end{eqnarray}
where we ignore the rescaling of $\delta$.
The corresponding triplet loss can be written as
\begin{eqnarray}\label{eq:trih}
\ell_{\mathrm{triplet}}(\x_i,\x_j,\x_k) = [\delta+ \x_i^{\top} \x_k - \x_i^{\top} \x_j ]_+
\end{eqnarray}

It is obvious from Eqn.~\ref{eq:const} that the number of total triplets can be cubic in the number of examples, which makes sampling inevitable for most of triplet based DML algorithms. 

With the unit length for both $\w$ and $\x$, the normalized SoftMax loss can be written as 
\begin{eqnarray}\label{eq:SoftMax}
\ell_{\mathrm{SoftMax_{norm}}}(\x_i) = -\log\frac{\exp(\lambda \w_{y_i}^\top \x_i)}{\sum_j \exp(\lambda \w_j^\top \x_i)}
\end{eqnarray}
where $\lambda$ is a scaling factor.

Surprisingly, we find that minimizing the normalized SoftMax loss with the smooth term $\lambda$ is equivalent to optimizing a smoothed triplet loss.
\begin{prop}\label{prop:1}
\begin{align}\label{eq:1}
&\ell_{\mathrm{SoftMax_{norm}}}(\x_i) = \max_{\p\in\Delta} \lambda\sum_j \p_j \x_i^\top(\w_j - \w_{y_i})+ H(\p)
\end{align}
where $\p\in\R^C$ is a distribution over classes and $\Delta$ is the simplex as $\Delta=\{\p|\sum_j \p_j=1,\forall j,\p_j\geq 0\}$. $H(\p)$ denotes the entropy of the distribution $\p$.
\end{prop}
\begin{proof}

According to the K.K.T. condition~\cite{boyd2004convex}, the distribution $\p$ in Eqn.~\ref{eq:1} has the closed-form solution
\[\p_j = \frac{\exp(\lambda \x_i^\top(\w_j - \w_{y_i}))}{\sum_j \exp(\lambda \x_i^\top(\w_j - \w_{y_i}))}\]
Therefore, we have
\small{
\begin{align*}
&\ell_{\mathrm{SoftMax_{norm}}}(\x_i) =\lambda\sum_j \p_j \x_i^\top(\w_j - \w_{y_i})+ H(\p)\\
&=\log(\sum_j \exp(\lambda \x_i^\top(\w_j - \w_{y_i})))= -\log\frac{\exp(\lambda \w_{y_i}^\top \x_i)}{\sum_j \exp(\lambda \w_j^\top \x_i)}
\end{align*}}
\end{proof}
\paragraph{Remark 1} Proposition~\ref{prop:1} indicates that the SoftMax loss optimizes the triplet constraints consisting of an original example and two centers, i.e., $(\x_i,\w_{y_i},\w_j)$. Compared with triplet constraints in Eqn.~\ref{eq:const}, the target of SoftMax loss is
\[\forall i, j, \quad \x_i^\top \w_{y_i} - \x_i^\top \w_{j} \geq 0\]
Consequently, the embeddings learned by minimizing SoftMax loss can be applicable for the distance-based tasks while it is designed for the classification task.
\paragraph{Remark 2} Without the entropy regularizer, the loss becomes
\[\max_{\p\in\Delta} \lambda\sum_j \p_j \x_i^\top\w_j -\lambda \x_i^\top\w_{y_i}\]
which is equivalent to
\[\max_j \{\x_i^\top\w_j\} - \x_i\w_{y_i}\]
Explicitly, it punishes the triplet with the most violation and becomes zero when the nearest neighbor of $\x_i$ is the corresponding center $\w_{y_i}$. The entropy regularizer reduces the influence from outliers and makes the loss more robust. $\lambda$ trades between the hardness of triplets and the regularizer. Moreover, minimizing the maximal entropy can make the distribution concentrated and further push the example away from irrelevant centers, which implies a large margin property.

\paragraph{Remark 3} Applying the similar analysis to the ProxyNCA loss~\cite{Attias17}: $\ell_{\mathrm{ProxyNCA}}(\x_i) = -\log\frac{\exp(\w_{y_i}^\top \x_i)}{\sum_{j\not =y_i} \exp(\w_j^\top \x_i)}$, we have
\begin{align*}
&\ell_{\mathrm{ProxyNCA}}(\x_i)  = \max_{\p\in\Delta} \lambda \sum_{j\not= y_i} \p_j \x_i^\top (\w_j  - \w_{y_i})+ H(\p)
\end{align*}
where $\p\in\R^{C-1}$. Compared with the SoftMax loss, it eliminates the benchmark triplet containing only the corresponding class center, which makes the loss unbounded. Our analysis suggests that the loss can be bounded as in Eqn.~\ref{eq:trih}: 
$\ell_{\mathrm{ProxyNCA}}^{\mathrm{hinge}}(\x_i) = [-\log\frac{\exp(\w_{y_i}^\top \x_i)}{\sum_{j\not =y_i} \exp(\w_j^\top \x_i)}]_+$. Validating the bounded loss is out of the scope of this work.

Despite optimizing SoftMax loss can learn the meaningful feature embeddings, the drawback is straightforward. It assumes that there is only a single center for each class while a real-world class can contain multiple local clusters due to the large intra-class variance as in Fig.~\ref{fig:illu}. The triplet constraints generated by conventional SoftMax loss is too brief to capture the complex geometry of the original data. Therefore, we introduce multiple centers for each class.

\subsection{Multiple Centers}
Now, we assume that each class has $K$ centers. Then, the similarity between the example $\x_i$ and the class $c$ can be defined as
\begin{eqnarray}\label{eq:simh}
\mathcal{S}_{i,c} = \max_k \x_i^\top \w_c^k
\end{eqnarray}
Note that other definitions of similarity can be applicable for this scenario (e.g., $\min_{\z\in \R^K} \|[\w_c^1,\cdots,\w_c^K] \z - \x_i\|_2$). We adopt a simple form to illustrate the influence of multiple centers.

With the definition of the similarity, the triplet constraint requires an example to be closer to its corresponding class than other classes
\[\forall j,\quad \mathcal{S}_{i,y_i} -\mathcal{S}_{i,j}\geq 0\]
As we mentioned above, minimizing the entropy term $H(\p)$ can help to pull the example to the corresponding center. To break the tie explicitly, we consider to introduce a small margin as in the conventional triplet loss in Eqn.~\ref{eq:const} and define the constraints as
\[\forall j_{j\not = y_i}, \quad \mathcal{S}_{i,y_{i}} - \mathcal{S}_{i,j}\geq \delta\]

By replacing the similarity in Eqn.~\ref{eq:1}, we can obtain the HardTriple loss as
\begin{align}\label{eq:hardtrip}
&\ell_{\mathrm{HardTriple}}(\x_i) = \max_{\p\in\Delta} \lambda\bigg(\sum_{j\not=y_i} \p_j (\mathcal{S}_{i,j} - (\mathcal{S}_{i,y_i}-\delta))\nonumber\\
&+\p_{y_i}(\mathcal{S}_{i,y_i} - \delta - (\mathcal{S}_{i,y_i}-\delta))\bigg)+ H(\p)\nonumber\\
&=-\log\frac{\exp(\lambda(\mathcal{S}_{i,y_i}-\delta))}{\exp(\lambda(\mathcal{S}_{i,y_i}-\delta))+\sum_{j\not=y_i}\exp(\lambda\mathcal{S}_{i,j})  }
\end{align}

HardTriple loss improves the SoftMax loss by providing multiple centers for each class. However, it requires the max operator to obtain the nearest center in each class while this operator is not smooth and the assignment can be sensitive between multiple centers. Inspired by the SoftMax loss, we can improve the robustness by smoothing the max operator.

Consider the problem
\[\max_k\x_i^\top\w_c^k\]
which is equivalent to
\begin{eqnarray}\label{eq:soft}
\max_{\q\in\Delta}\sum_k \q_k \x_i^\top \w_c^k
\end{eqnarray}
we add the entropy regularizer to the distribution $\q$ as
\[\max_{\q\in\Delta}\sum_k \q_k \x_i^\top \w_c^k+\gamma H(\q)\]
With a similar analysis as in Proposition~\ref{prop:1}, $\q$ has the closed-form solution as
\[\q_k = \frac{\exp(\frac{1}{\gamma}\x_i^\top \w_c^k)}{\sum_k \exp(\frac{1}{\gamma}\x_i^\top \w_c^k)}\]
Taking it back to the Eqn.~\ref{eq:soft}, we define the relaxed similarity between the example $\x_i$ and the class $c$ as
\[\mathcal{S}'_{i,c} = \sum_k  \frac{\exp(\frac{1}{\gamma}\x_i^\top \w_c^k)}{\sum_k \exp(\frac{1}{\gamma}\x_i^\top \w_c^k)}\x_i^\top \w_c^k\]

By applying the smoothed similarity, we define the SoftTriple loss as
\begin{align}\label{eq:softtrip}
&\ell_{\mathrm{SoftTriple}}(\x_i) \nonumber\\
&= -\log\frac{\exp(\lambda(\mathcal{S}'_{i,y_i}-\delta))}{\exp(\lambda(\mathcal{S}'_{i,y_i}-\delta))+\sum_{j\not=y_i}\exp(\lambda\mathcal{S}'_{i,j})  }
\end{align}

Fig.~\ref{fig:illu2} illustrates the differences between the SoftMax loss and the proposed losses. 
\begin{figure}[!ht]
\centering
\includegraphics[width=3.2in]{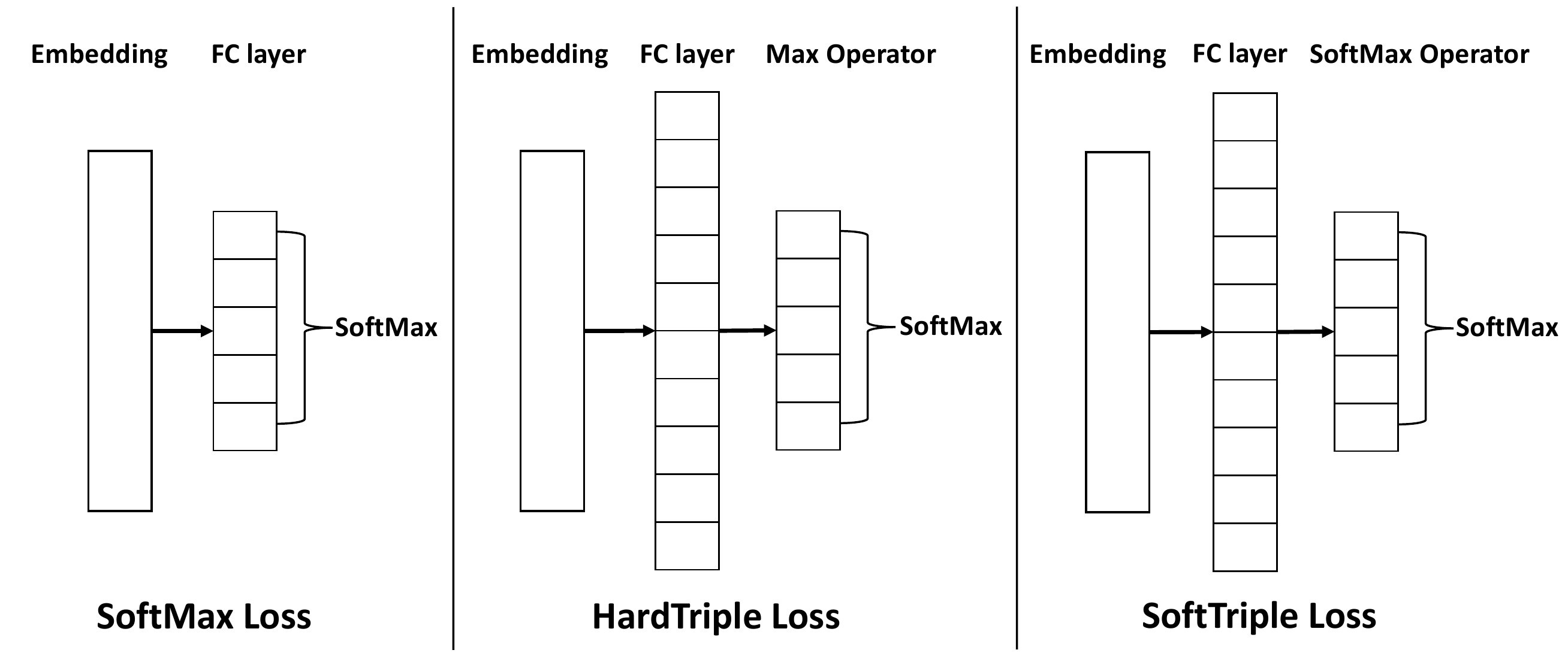}
\caption{Illustration of differences between SoftMax loss and proposed losses. Compared with the SoftMax loss, we first increase the dimension of the FC layer to include multiple centers for each class (e.g., 2 centers per class in this example). Then, we obtain the similarity for each class by different operators. Finally, the distribution over different classes is computed with the similarity obtained from each class.\label{fig:illu2}}
\end{figure}

Finally, we will show that the strategy of applying centers to construct triplet constraints can recover the constraints on original triplets. 
\begin{thm}\label{thm:1}
Given two examples $\x_i$ and $\x_j$ that are from the same class and have the same nearest center and $\x_k$ is from a different class, if the triple constant containing centers is satisfied
\[\x_i^\top\w_{y_i} - \x_i^\top\w_{y_k}\geq \delta\]
and we assume $\forall i, \|\x_i-\w_{y_i}\|_2\leq \epsilon$, then we have
\[\x_i^\top \x_j - \x_i^\top\x_k\geq \delta-2\epsilon \]
\end{thm}
\begin{proof}
\begin{eqnarray*}
&&\x_i^\top \x_j - \x_i^\top\x_k  = \x_i^\top (\x_j-\w_{y_i})+\x_i^\top\w_{y_i} - \x_i^\top\x_k\\
&&\geq \x_i^\top (\x_j-\w_{y_i})+\x_i^\top(\w_{y_k} - \x_k)+\delta\\
&&\geq \delta - \|\x_i\|_2\|\x_j-\w_{y_i}\|_2-\|\x_i\|_2\|\w_{y_k} - \x_k\|_2\\
&&=\delta - \|\x_j-\w_{y_i}\|_2-\|\w_{y_k} - \x_k\|_2 \geq \delta - 2\epsilon
\end{eqnarray*}
\end{proof}

Theorem~\ref{thm:1} demonstrates that optimizing the triplets consisting of centers with a margin $\delta$ can reserve the large margin property on the original triplet constraints. It also implies that more centers can be helpful to reduce the intra-class variance $\epsilon$. In the extreme case that the number of centers is equal to the number of examples, $\epsilon$ becomes zero. However, adding more centers will increase the size of the last fully connected layer and make the optimization slow and computation expensive. Besides, it may incur the overfitting problem.

Therefore, we have to choose an appropriate number of centers for each class that can have a small approximation error while keeping a compact set of centers. We will demonstrate the strategy in the next subsection.

\subsection{Adaptive Number of Centers}

Finding an appropriate number of centers for data is a challenging problem that also appears in unsupervised learning, e.g.,  clustering. The number of centers $K$ trades between the efficiency and effectiveness. In conventional DML algorithms, $K$ equals to the number of original examples. It makes the number of total triplet constraints up to cubic of the number of original examples. In SoftMax loss, $K=1$ reduces the number of constraints to be linear in the number of original examples, which is efficient but can be ineffective. Without the prior knowledge about the distribution of each class, it is hard to set $K$ precisely. 

Different from the strategy of setting the appropriate $K$ for each class, we propose to set a sufficiently large $K$ and then encourage similar centers to merge with each other. It can keep the diversity in the generated centers while shrinking the number of unique centers.

For each center $\w_j^t$, we can generate a matrix as
\[M_j^t = [\w_j^1-\w_j^t,\cdots,\w_j^K-\w_j^t]^\top\]
If $\w_j^s$ and $\w_j^t$ are similar, they can be collapsed to be the same one such that $\|\w_j^s-\w_j^t\|_2=0$, which is the $L_2$ norm of the $s$-th row in the matrix $M_j^t$. Therefore, we regularize the $L_2$ norm of rows in $M_j^t$ to obtain a sparse set of centers, which can be written as the $L_{2,1}$ norm
\[\|M_j^t\|_{2,1} = \sum_s^K \|\w_j^s - \w_j^t\|_2\]

By accumulating $L_{2,1}$ norm over multiple centers, we can have the regularizer for the $j$-th class as 
\[R(\w_j^1,\cdots,\w_j^K) = \sum_t^K\|M_j^t\|_{2,1}\]
Since $\w$ has the unit length, the regularizer is simplified as
\begin{eqnarray}\label{eq:reg}
&&R(\w_j^1,\cdots,\w_j^K) = \sum_{t=1}^K \sum_{s=t+1}^K \sqrt{2-2\w_j^{s\top}\w_j^t}
\end{eqnarray}

With the regularizer, our final objective becomes
\begin{eqnarray}\label{eq:object}
\min \frac{1}{N}\sum_i\ell_{\mathrm{SoftTriple}}(\x_i) +\frac{\tau\sum_j^CR(\w_j^1,\cdots,\w_j^K)}{CK(K-1)}
\end{eqnarray}
where $N$ is the number of total examples.

\section{Experiments}
\label{sec:exp}
We conduct experiments on three benchmark fine-grained visual categorization data sets: \textit{CUB-2011}, \textit{Cars196} and \textit{SOP}. We follow the settings in other works~\cite{GeHDS18,Attias17} for the fair comparison. Specifically, we adopt the Inception~\cite{SzegedyLJSRAEVR15} with the batch normalization~\cite{IoffeS15} as the backbone architecture. The parameters of the backbone are initialized with the model trained on the ImageNet ILSVRC 2012 data set~\cite{ILSVRC15} and then fine-tuned on the target data sets. The images are cropped to $224\times 224$ as the input of the network. During training, only random horizontal mirroring and random crop are used as the data augmentation. A single center crop is taken for test. The model is optimized by Adam with the batch size as 32 and the number of epochs as $50$. The initial learning rates for the backbone and centers are set to be $1e\mbox{-}4$ and $1e\mbox{-}2$, respectively. Then, they are divided by $10$ at $\{20, 40\}$ epochs. Considering that images in \textit{CUB-2011} and \textit{Cars196} are similar to those in ImageNet, we freeze BN on these two data sets and keep BN training on the rest one. Embeddings of examples and centers have the unit length in the experiments.

We compare the proposed triplet loss to the normalized SoftMax loss. The SoftMax loss in Eqn.~\ref{eq:SoftMax} is denoted as \textbf{SoftMax$_\mathrm{norm}$}. We refer the objective in Eqn.~\ref{eq:object} as \textbf{SoftTriple}. We set $\tau = 0.2$ and $\gamma=0.1$ for SoftTriple. Besides, we set a small margin as $\delta= 0.01$ to break the tie explicitly. The number of centers is set to $K=10$.

We evaluate the performance of the learned embeddings from different methods on the tasks of retrieval and clustering. For retrieval task, we use the Recall@$k$ metric as in \cite{SongXJS16}. The quality of clustering is measured by the Normalized Mutual Information ($\mathrm{NMI}$)~\cite{manning2010}. Given the clustering assignment $\mathbb{C} = \{c_1,\cdots,c_n\}$ and the ground-truth label $\Omega=\{y_1,\cdots,y_n\}$, NMI is computed as
$\mathrm{NMI} = \frac{2 I(\Omega;\mathcal{\mathbb{C}})}{H(\Omega)+H(\mathcal{\mathbb{C})}}$, where $I(\cdot,\cdot)$ measures the mutual information and $H(\cdot)$ denotes the entropy.

\subsection{CUB-2011}
First, we compare the methods on a fine-grained birds data set \textit{CUB-2011}~\cite{WahCUB_200_2011}. It consists of $200$ species of birds and $11,788$ images. Following the common practice, we split the data set as that the first $100$ classes are used for training and the rest are used for test. We note that different works report the results with different dimension of embeddings while the size of embeddings has a significant impact on the performance. For fair comparison, we report the results for the dimension of $64$, which is adopted by many existing methods and the results with $512$ feature embeddings, which reports the state-of-the-art results on most of data sets. 

\begin{figure*}[!ht]
\centering
\includegraphics[width=4.5in]{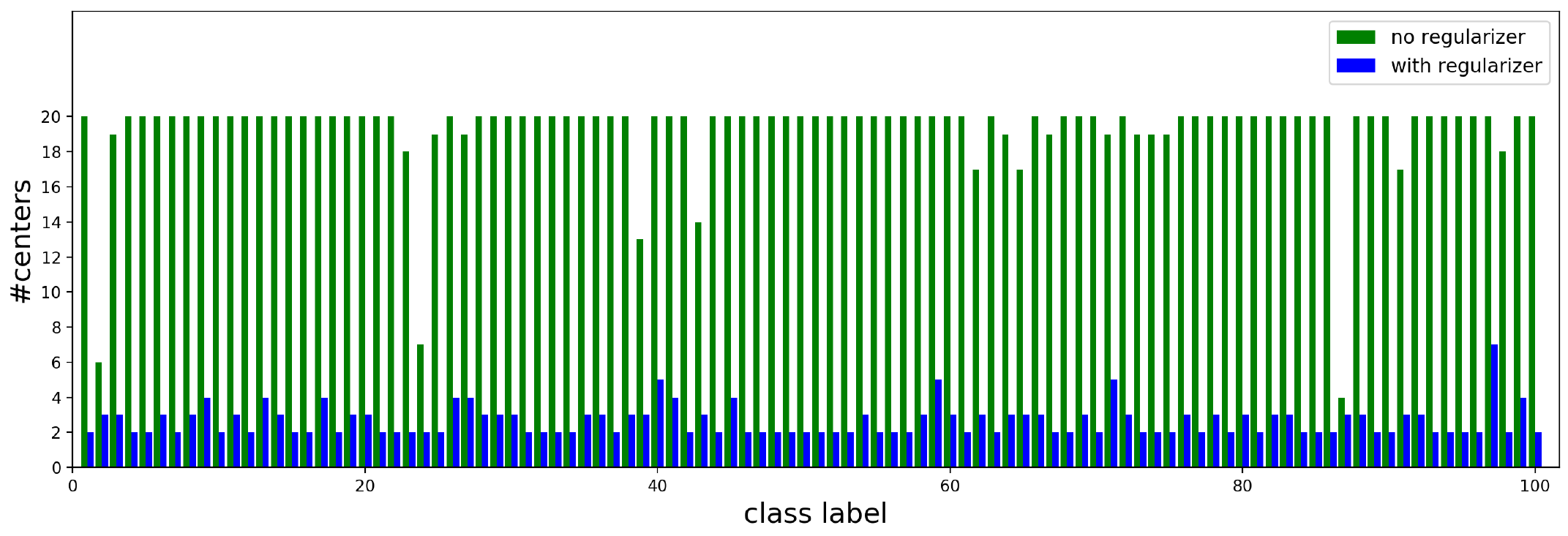}
\caption{Comparison of the number of unique centers in each class on \textit{CUB-2011}. The initial number of centers is set to $20$.\label{fig:classn}}
\end{figure*}

Table~\ref{ta:birds64} summarizes the results with $64$ embeddings. Note that Npairs$^*$ applies the multi-scale test while all other methods take a single crop test. For SemiHard~\cite{SchroffKP15}, we report the result recorded in \cite{SongJR017}. First, it is surprising to observe that the performance of SoftMax$_\mathrm{norm}$ surpasses that of the existing metric learning methods. It is potentially due to the fact that SoftMax loss optimizes the relations of examples as a smoothed triplet loss, which is analyzed in Proposition~\ref{prop:1}. Second, SoftTriple demonstrates the best performance among all benchmark methods. Compared to ProxyNCA, SoftTriple improves the state-of-the-art performance by $10\%$ on R@1. Besides, it is $2\%$ better than SoftMax$_{\mathrm{norm}}$. It verifies that SoftMax loss cannot capture the complex geometry of real-world data set with a single center for each class. When increasing the number of centers, SoftTriple can depict the inherent structure of data better. Finally, both of SoftMax and SoftTriple show the superior performance compared to existing methods. It demonstrates that meaningful embeddings can be learned without a sampling phase.

\begin{table}[!ht]
\centering
\small
\caption{Comparison on \textit{CUB-2011}. The dimension of the embeddings for all methods is $64$.}\label{ta:birds64}
\begin{tabular}{c|ccccc}
Methods&R@1&R@2&R@4&R@8&NMI\\\hline
SemiHard~\cite{SchroffKP15}&42.6&55.0&66.4&77.2&55.4\\
LiftedStruct~\cite{SongXJS16}&43.6&56.6&68.6&79.6&56.5\\
Clustering~\cite{SongJR017}&48.2&61.4&71.8&81.9&59.2\\
Npairs$^*$~\cite{Sohn16}&51.0&63.3&74.3&83.2&60.4\\
ProxyNCA~\cite{Attias17}&49.2&61.9&67.9&72.4&59.5\\\hline
SoftMax$_{\mathrm{norm}}$&57.8&70.0&80.1&87.9&65.3\\
SoftTriple&\textbf{60.1}&\textbf{71.9}&\textbf{81.2}&\textbf{88.5}&\textbf{66.2}\\
\end{tabular}
\end{table}

Table~\ref{ta:birds512} compares SoftTriple with $512$ embeddings to the methods with large embeddings. HDC~\cite{YuanYZ17} applies the dimension as $384$. Margin~\cite{ManmathaWSK17} takes $128$ dimension of embeddings and uses ResNet50~\cite{HeZRS16} as the backbone. HTL~\cite{GeHDS18} sets the dimension of embeddings to $512$ and reports the state-of-the-art result on the backbone of Inception. With the large number of embeddings, it is obvious that all methods outperform existing DML methods with $64$ embeddings in Table~\ref{ta:birds64}. It is as expected since the high dimensional space can separate examples better, which is consistent with the observation in other work~\cite{SongXJS16}. Compared with other methods, the R@1 of SoftTriple improves more than $8\%$ over HTL that has the same backbone as SoftTriple. It also increases R@1 by about $2\%$ over Margin, which applies a stronger backbone than Inception. It shows that SoftTriple loss is applicable with large embeddings.

\begin{table}[!ht]
\centering
\small
\caption{Comparison on \textit{CUB-2011} with large embeddings. ``-'' means the result is not available.}\label{ta:birds512}
\begin{tabular}{c|ccccc}
Methods&R@1&R@2&R@4&R@8&NMI\\\hline
HDC~\cite{YuanYZ17}&53.6&65.7&77.0&85.6&-\\
Margin~\cite{ManmathaWSK17}&63.6&74.4&83.1&90.0&69.0\\
HTL~\cite{GeHDS18}&57.1&68.8&78.7&86.5&-\\\hline
SoftMax$_{\mathrm{norm}}$&64.2&75.6&84.3&90.2&68.3\\
SoftTriple&\textbf{65.4}&\textbf{76.4}&\textbf{84.5}&\textbf{90.4}&\textbf{69.3}\\
\end{tabular}
\end{table}

To validate the effect of the proposed regularizer, we compare the number of unique centers for each class in Fig.~\ref{fig:classn}. We set a larger number of centers as $K=20$ to make the results explicit and then run SoftTriple with and without the regularizer in Eqn.~\ref{eq:reg}. Fig.~\ref{fig:classn} illustrates that the one without regularizer will hold a set of similar centers. In contrast, SoftTriple with the regularizer can shrink the size of centers significantly and make the optimization effective. 

Besides, we demonstrate the R@1 of SoftTriple with varying the number of centers in Fig.~\ref{fig:recall}. Red line denotes SoftTriple loss equipped with the regularizer while blue dashed line has no regularizer. We find that when increasing the number of centers from $1$ to $10$, the performance of SoftTriple is improved significantly, which confirms that with leveraging multiple centers, the learned embeddings can capture the data distribution better. If adding more centers, the performance of SoftTriple almost remains the same and it shows that the regularizer can help to learn the compact set of centers and will not be influenced by the initial number of centers. On the contrary, without the regularizer, the blue dashed line illustrates that the performance will degrade due to overfitting when the number of centers are over-parameterized.

\begin{figure}[!ht]
\centering
\includegraphics[width=2in]{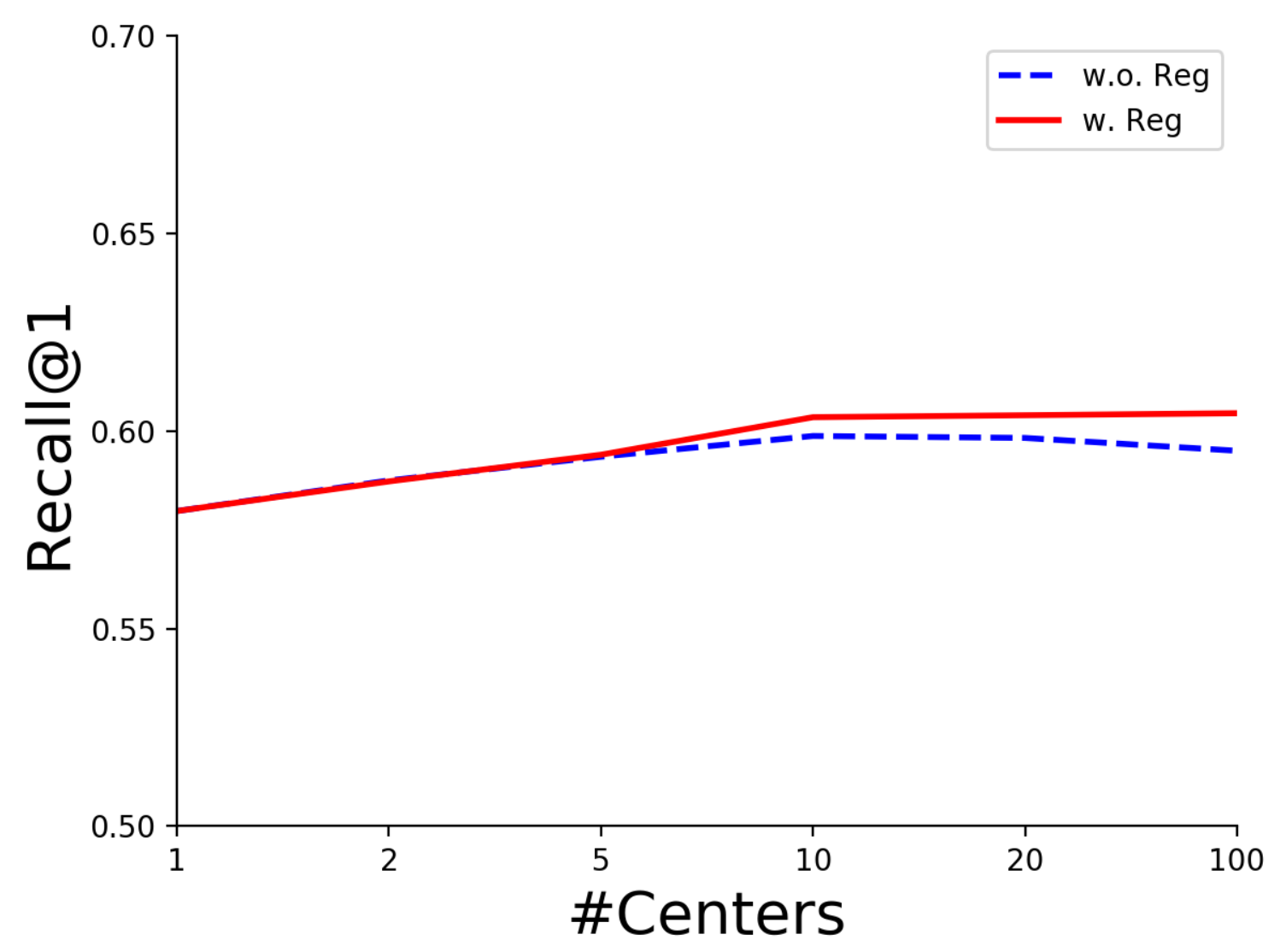}
\caption{Illustration of SoftTriple with different number of centers and the influence of the regularizer. With the proposed regularizer as denoted by the red line, the performance is stable to the initial number of centers $K$ when it is sufficiently large.\label{fig:recall}}
\end{figure}

Finally, we illustrate the examples of retrieved images in Fig.~\ref{fig:exa}. The first column indicates the query image. The columns $2$-$4$ show the most similar images retrieved according to the embeddings learned by SoftMax$_\mathrm{norm}$. The last four columns are the similar images returned by using the embeddings from SoftTriple. Evidently, embeddings from SoftMax$_\mathrm{norm}$ can obtain the meaningful neighbors while the objective is for classification. Besides, SoftTriple improves the performance and can eliminate the images from different classes among the top of retrieved images, which are highlighted with red bounding boxes in SoftMax$_\mathrm{norm}$.

\begin{figure}[!ht]
\centering
\includegraphics[width=3in]{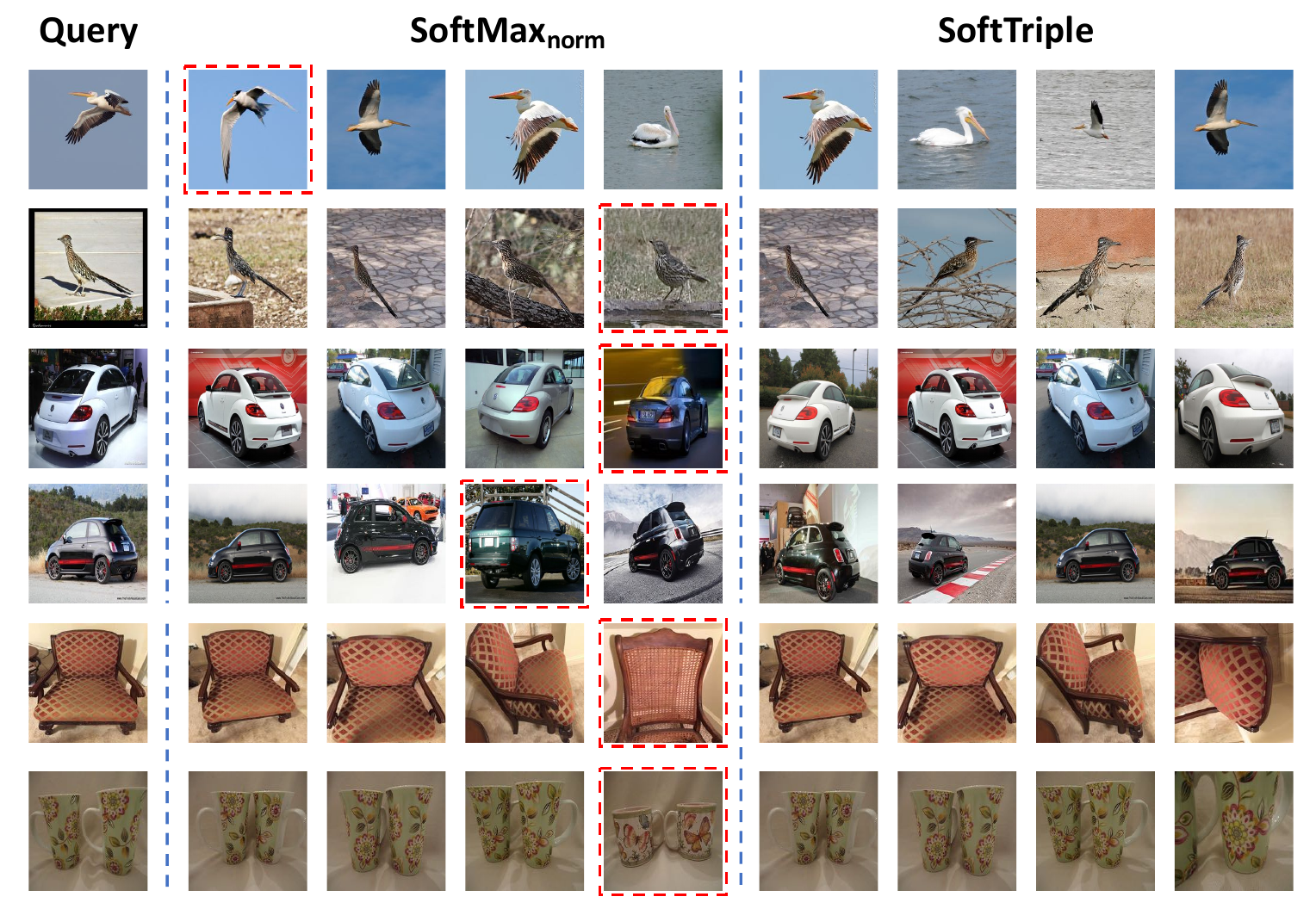}
\caption{Examples of retrieved most similar images with the learned embeddings from SoftMax$_\mathrm{norm}$ and SoftTriple. The images from the classes that are different from the query image are highlighted by red bounding boxes.}\label{fig:exa}
\end{figure}

\subsection{Cars196}

Then, we conduct the experiments on \textit{Cars196} data set~\cite{3DRR2013}, which contains $196$ models of cars and $16,185$ images. We use the first $98$ classes for training and the rest for test. Table~\ref{ta:cars64} summaries the performance with $64$ embeddings. The observation is similar as for \textit{CUB-2011}. SoftMax$_{\mathrm{norm}}$ shows the superior performance and is $3\%$ better than ProxyNCA on R@1. Additionally, SoftTriple can further improve the performance by about $2\%$, which demonstrates the effectiveness of the proposed loss function.

\begin{table}[!ht]
\centering
\small
\caption{Comparison on \textit{Cars196}. The dimension is $64$.}\label{ta:cars64}
\begin{tabular}{c|ccccc}
Methods&R@1&R@2&R@4&R@8&NMI\\\hline
SemiHard~\cite{SchroffKP15}&51.5&63.8&73.5&82.4&53.4\\
LiftedStruct~\cite{SongXJS16}&53.0&65.7&76.0&84.3&56.9\\
Clustering~\cite{SongJR017}&58.1&70.6&80.3&87.8&59.0\\
Npairs$^*$~\cite{Sohn16}&71.1&79.7&86.5&91.6&64.0\\
ProxyNCA~\cite{Attias17}&73.2&82.4&86.4&88.7&64.9\\\hline
SoftMax$_{\mathrm{norm}}$&76.8&85.6&91.3&95.2&66.7\\
SoftTriple&\textbf{78.6}&\textbf{86.6}&\textbf{91.8}&\textbf{95.4}&\textbf{67.0}\\
\end{tabular}
\end{table}

In Table~\ref{ta:cars512}, we present the comparison with large dimension of embeddings. The number of embeddings for all methods in the comparison is the same as described in the experiments on \textit{CUB-2011}. On this data set, HTL~\cite{GeHDS18} reports the state-of-the-art result while SoftTriple outperforms it and increases R@1 by $3\%$. 

\begin{table}[!ht]
\centering
\small
\caption{Comparison on \textit{Cars196} with large embeddings.}\label{ta:cars512}
\begin{tabular}{c|ccccc}
Methods&R@1&R@2&R@4&R@8&NMI\\\hline
HDC~\cite{YuanYZ17}&73.7&83.2&89.5&93.8&-\\
Margin~\cite{ManmathaWSK17}&79.6&86.5&91.9&95.1&69.1\\
HTL~\cite{GeHDS18}&81.4&88.0&92.7&95.7&-\\\hline
SoftMax$_{\mathrm{norm}}$&83.2&89.5&94.0&96.6&69.7\\
SoftTriple&\textbf{84.5}&\textbf{90.7}&\textbf{94.5}&\textbf{96.9}&\textbf{70.1}\\
\end{tabular}
\end{table}

\subsection{Stanford Online Products}

Finally, we evaluate the performance of different methods on the Stanford Online Products (\textit{SOP}) data set~\cite{SongXJS16}. It contains $120,053$ product images downloaded from eBay.com and includes $22,634$ classes. We adopt the standard splitting, where $11,318$ classes are used for training and the rest for test. Note that each class has about $5$ images, so we set $K=2$ for this data set and discard the regularizer. We also increase the initial learning rate for centers from $0.01$ to $0.1$.

We first report the results with $64$ embeddings in Table~\ref{ta:sop64}. In this comparison, SoftMax$_{\mathrm{norm}}$ is $2\%$ better than ProxyNCA on R@1. By simply increasing the number of centers from $1$ to $2$, we observe that SoftTriple gains another $0.4\%$ on R@1. It confirms that multiple centers can help to capture the data structure better. 

Table~\ref{ta:sop512} states the performance with large embeddings. We can get a similar conclusion as in Table~\ref{ta:sop64}. Both SoftMax$_{\mathrm{norm}}$ and SoftTriple outperform the state-of-the-art methods. SoftTriple improves the state-of-the-art by more than $3\%$ on R@1. It demonstrates the advantage of learning embeddings without sampling triplet constraints.

\begin{table}[!ht]
\centering
\small
\caption{Comparison on \textit{SOP}. The dimension is $64$.}\label{ta:sop64}
\begin{tabular}{c|cccc}
Methods&R@1&R@10&R@100&NMI\\\hline
SemiHard~\cite{SchroffKP15}&66.7&82.4&91.9&89.5\\
LiftedStruct~\cite{SongXJS16}&62.5&80.8&91.9&88.7\\
Clustering~\cite{SongJR017}&67.0&83.7&93.2&89.5\\
ProxyNCA~\cite{Attias17}&73.7&-&-&90.6\\\hline
SoftMax$_{\mathrm{norm}}$&75.9&88.8&95.2&91.5\\
SoftTriple&\textbf{76.3}&\textbf{89.1}&\textbf{95.3}&\textbf{91.7}\\
\end{tabular}
\end{table}

\begin{table}[!ht]
\centering
\small
\caption{Comparison on \textit{SOP} with large embeddings.}\label{ta:sop512}
\begin{tabular}{c|cccc}
Methods&R@1&R@10&R@100&NMI\\\hline
Npairs$^*$~\cite{Sohn16}&67.7&83.8&93.0&88.1\\
HDC~\cite{YuanYZ17}&69.5&84.4&92.8&-\\
Margin~\cite{ManmathaWSK17}&72.7&86.2&93.8&90.7\\
HTL~\cite{GeHDS18}&74.8&88.3&94.8&-\\\hline
SoftMax$_{\mathrm{norm}}$&78.0&90.2&\textbf{96.0}&91.9\\
SoftTriple&\textbf{78.3}&\textbf{90.3}&95.9&\textbf{92.0}\\
\end{tabular}
\end{table}

\section{Conclusion}
\label{sec:conclud}
Sampling triplets from a mini-batch of data can degrade the performance of deep metric learning due to its poor coverage over the whole data set. To address the problem, we propose the novel SoftTriple loss to learn the embeddings without sampling. By representing each class with multiple centers, the loss can be optimized with triplets defined with the similarities between the original examples and classes. Since centers are encoded in the last fully connected layer, we can learn embeddings with the standard SGD training pipeline for classification and eliminate the sampling phase. The consistent improvement from SoftTriple over fine-grained benchmark data sets confirms the effectiveness of the proposed loss function. Since SoftMax loss is prevalently applied for classification, SoftTriple loss can also be applicable for that. Evaluating SoftTriple on the classification task can be our future work.

{\small
\bibliographystyle{ieee_fullname}
\bibliography{stri}
}

\end{document}